\documentclass{article}

\usepackage{arxiv}
\usepackage{amsmath,amssymb,amsthm}
\usepackage{graphicx}
\usepackage{algorithm}
\usepackage{algpseudocode}
\usepackage{booktabs}
\usepackage{cite}
\usepackage{multirow}
\usepackage{xcolor}

\title{PolyKAN: A Polyhedral Analysis Framework for Provable and Approximately Optimal KAN Compression}
\author{
	Di Zhang \\
	School of Advanced Technology \\
	Xi'an Jiaotong-Liverpool University \\
	Suzhou, Jiangsu, China \\
	\texttt{di.zhang@xjtlu.edu.cn}
}

\newtheorem{theorem}{Theorem}
\newtheorem{lemma}{Lemma}
\newtheorem{definition}{Definition}

\newtheorem{problem}{Problem}

\begin{document}
	
	\maketitle
	
	\begin{abstract}
		Kolmogorov-Arnold Networks (KANs) have emerged as a promising alternative to traditional Multi-Layer Perceptrons (MLPs), offering enhanced interpretability and a solid mathematical foundation. However, their parameter efficiency remains a significant challenge for practical deployment. This paper introduces \textbf{PolyKAN}, a novel theoretical framework for KAN compression that provides formal guarantees on both model size reduction and approximation error. By leveraging the inherent piecewise polynomial structure of KANs, we formulate the compression problem as a polyhedral region merging task. We establish a rigorous polyhedral characterization of KANs, develop a complete theory of $\epsilon$-equivalent compression, and design a dynamic programming algorithm that achieves approximately optimal compression under specified error bounds. Our theoretical analysis demonstrates that PolyKAN achieves provably near-optimal compression while maintaining strict error control, with guaranteed global optimality for univariate spline functions. This framework provides the first formal foundation for KAN compression with mathematical guarantees, opening new directions for the efficient deployment of interpretable neural architectures.
		
	\end{abstract}
	
	\section{Introduction}
	
	The pursuit of neural network architectures that combine interpretability with strong performance has renewed interest in mathematically grounded alternatives to traditional black-box models. Kolmogorov-Arnold Networks (KANs) \cite{liu2025kan}, inspired by the celebrated Kolmogorov-Arnold Representation Theorem \cite{kolmogorov1957representation}, have recently emerged as a compelling framework that replaces fixed activation functions with learnable spline transformations. Although KANs demonstrate remarkable interpretability and empirical performance in certain function approximation tasks, their parameter efficiency presents a major obstacle to practical application, as each network connection requires an independent spline function with multiple parameters.
	
	Existing neural network compression methods—including pruning \cite{han2015learning}, knowledge distillation \cite{hinton2015distilling}, and architecture search—largely rely on heuristic strategies and lack formal guarantees. For KANs in particular, the absence of theoretically grounded compression methods represents a significant gap in the literature. However, the spline-based roots of KANs offer unique opportunities for rigorous analysis through polyhedral theory, as each spline naturally partitions its input domain into piecewise polynomial regions.
	
	This paper presents PolyKAN, a comprehensive theoretical framework for KAN compression with provable guarantees. Our work makes three fundamental contributions: First, we establish a complete polyhedral characterization of KANs, demonstrating that their input space partitions exhibit highly regular axis-aligned structures—in stark contrast to the arbitrary polyhedral complexes of ReLU networks. This structural regularity enables rigorous theoretical analysis. Second, we develop a formal theory of $\epsilon$-equivalent compression for KANs, providing necessary and sufficient conditions for region merging while preserving strict error bounds. Our analysis includes precise bounds on error propagation through multiple network layers. Third, we design and analyze a dynamic programming algorithm that guarantees optimal compression for univariate spline functions and delivers approximately optimal guarantees for the entire network, with provable polynomial time complexity.
	
	The PolyKAN framework transforms KAN compression from a heuristic process into a mathematically rigorous optimization problem with verifiable guarantees. By connecting spline theory, polyhedral geometry, and algorithm design, we lay the foundation for trustworthy compression of interpretable neural architectures.
	
	\section{Background and Related Work}
	
	\subsection{Kolmogorov-Arnold Networks and Spline Theory}
	
	The Kolmogorov-Arnold Representation Theorem \cite{kolmogorov1957representation} states that any multivariate continuous function can be represented as a composition of a finite number of univariate functions and additions. KANs \cite{liu2025kan} implement this theorem by placing learnable univariate functions (typically B-splines) on the network edges while restricting nodes to summation operations. This architectural shift from traditional MLPs provides advantages in both mathematical interpretability and empirical performance in certain function approximation tasks.
	
	Formally, a KAN layer transforms its input $\mathbf{x} \in \mathbb{R}^{n}$ to output $\mathbf{y} \in \mathbb{R}^{m}$ through the relation:
	\begin{equation}
		y_j = \sum_{i=1}^{n} s_{j,i}(x_i), \quad j=1,\ldots,m
	\end{equation}
	where each $s_{j,i}: \mathbb{R} \to \mathbb{R}$ is a spline function parameterized by knot locations and polynomial coefficients.
	
	The approximation properties of splines are well-understood \cite{schumaker2007spline} and have deep connections to function space theory. Recent work \cite{liu2025kan} has begun to explore the implications of these properties for neural network design, but a comprehensive theoretical framework for KAN optimization remains underdeveloped.
	
	\subsection{Neural Network Compression}
	
	Model compression techniques aim to reduce the computational and memory requirements of neural networks without significant performance degradation. Pruning methods \cite{han2015learning, frankle2020lottery} remove parameters or connections based on various importance criteria, while knowledge distillation \cite{hinton2015distilling} trains compact student networks to mimic larger teacher models. Neural architecture search \cite{zoph2016neural} automates the design of efficient network structures.
	
	Despite empirical success, most of these methods lack formal guarantees. Theoretical work on network compression typically focuses on simplified settings \cite{arora2018theoretical} or provides only asymptotic guarantees. The lottery ticket hypothesis \cite{frankle2020lottery} offers intriguing insights but fails to provide constructive compression algorithms with bounded error.
	
	\subsection{Polyhedral Theory in Deep Learning}
	
	Polyhedral theory provides powerful tools for analyzing piecewise linear neural networks. Substantial work has studied the linear regions of ReLU networks \cite{montufar2014number, serra2018bounding, de2021piecewise}, establishing connections between region counts and network expressive power. Mixed-integer programming formulations \cite{anderson2020strong} have been developed for verifying properties of ReLU networks.
	
	Recent work has begun to explore connections between KANs and ReLU networks. \cite{schoots2025relating} established formal relationships between piecewise linear KANs and ReLU networks, proving that under certain conditions, both architectures exhibit similar polyhedral complexity. Meanwhile, \cite{qiu2024relu} introduced ReLU-KANs, a variant that replaces spline functions with ReLU activations while maintaining the Kolmogorov-Arnold structure, further bridging the gap between these architectures.
	
	However, the polyhedral structure of general KANs differs fundamentally from that of ReLU networks. While ReLU networks partition their input space with hyperplanes of arbitrary orientation, KAN partitions are axis-aligned due to their spline-based construction. This structural regularity makes KANs particularly amenable to polyhedral analysis and enables the development of compression algorithms with strong theoretical guarantees.
	
	\section{Polyhedral Characterization of KANs}
	
	\subsection{Univariate Spline Partitions}
	
	We begin by formalizing the polyhedral structure of individual spline functions, which constitute the building blocks of KANs.
	
	\begin{definition}[Spline Polyhedron]
		For a B-spline function $s: [a,b] \to \mathbb{R}$ defined on interval $[a,b]$ with knot sequence $t_0, t_1, \ldots, t_k$ where $a = t_0 < t_1 < \cdots < t_k = b$, its \textbf{spline polyhedron} is a tuple $\mathcal{P}_s = (R, P)$ where:
		\begin{itemize}
			\item $R = \{R_1, R_2, \ldots, R_k\}$ is a collection of regions with $R_i = \{x \in \mathbb{R} : t_{i-1} \leq x \leq t_i\}$
			\item $P = \{p_1, p_2, \ldots, p_k\}$ is a collection of polynomial functions with $p_i$ defined on $R_i$
		\end{itemize}
	\end{definition}
	
	This definition captures the piecewise polynomial nature of spline functions, where each region corresponds to an interval between consecutive knots.
	
	\subsection{Multilayer KAN Partitions}
	
	For multilayer KANs, the polyhedral structure arises from the composition of spline functions across layers. Consider an $L$-layer KAN with architecture $[n_0, n_1, \ldots, n_L]$, where $n_l$ denotes the number of nodes in layer $l$.
	
	\begin{theorem}[KAN Polyhedral Region Structure]
		The input space partition of an $L$-layer KAN is a refinement of the partitions induced by spline functions across all layers. The total number of linear regions satisfies:
		\begin{equation}
			N_{\text{regions}} \leq \prod_{l=1}^L \prod_{i=1}^{n_{l-1}} \prod_{j=1}^{n_l} (k^{(l)}_{j,i} - 1)
		\end{equation}
		where $k^{(l)}_{j,i}$ is the number of knots in the spline function $s^{(l)}_{j,i}$ connecting node $i$ in layer $l-1$ to node $j$ in layer $l$.
	\end{theorem}
	
	\begin{proof}
		Each spline function $s^{(l)}_{j,i}$ partitions its one-dimensional input space into at most $(k^{(l)}_{j,i} - 1)$ intervals. Since these partitions act independently along different dimensions at each layer, the Cartesian product of these partitions produces the region structure of that layer's output space. Composition across layers refines these partitions, leading to the product bound.
	\end{proof}
	
	\subsection{Structural Properties of KAN Polyhedra}
	
	The polyhedral structure of KANs exhibits three fundamental properties that distinguish them from ReLU networks and enable our compression theory.
	
	\begin{lemma}[Axis-Alignedness]
		All boundary hyperplanes of KAN polyhedral regions are axis-aligned. That is, each boundary can be expressed as $x_d = c$ for some coordinate $d$ and constant $c$.
	\end{lemma}
	
	\begin{proof}
		The knots of each spline function $s^{(l)}_{j,i}(x_i)$ define partitions of the form $x_i = t_m$, which are hyperplanes perpendicular to the $i$-th coordinate axis. The composition of such axis-aligned partitions preserves axis-alignedness.
	\end{proof}
	
	\begin{lemma}[Rectangular Structure]
		Each KAN polyhedral region is an axis-aligned rectangle (Cartesian product of intervals).
	\end{lemma}
	
	\begin{proof}
		By the axis-aligned property, intersections of half-spaces defined by axis-aligned hyperplanes necessarily produce rectangular regions.
	\end{proof}
	
	\begin{lemma}[Function Regularity]
		Within each polyhedral region, a KAN is a smooth multivariate polynomial function. Across region boundaries, KANs maintain continuity (for linear B-splines) or higher-order smoothness (for higher-degree B-splines).
	\end{lemma}
	
	These structural properties significantly simplify polyhedral analysis compared to ReLU networks, whose regions can be arbitrary convex polyhedra with boundaries of arbitrary orientation.
	
	\section{Theory of Provable KAN Compression}
	
	\subsection{Formal Problem Statement}
	
	We now formalize the problem of KAN compression with provable guarantees.
	
	\begin{definition}[$\epsilon$-Equivalent Compression]
		Given a KAN network $\mathcal{N}: \mathcal{X} \to \mathbb{R}$ and an error tolerance $\epsilon > 0$, a compressed network $\mathcal{N}'$ is an \textbf{$\epsilon$-equivalent compression} of $\mathcal{N}$ if:
		\begin{equation}
			\|\mathcal{N} - \mathcal{N}'\|_\infty = \sup_{x \in \mathcal{X}} |\mathcal{N}(x) - \mathcal{N}'(x)| \leq \epsilon
		\end{equation}
	\end{definition}
	
	\begin{problem}[Optimal KAN Compression]
		Given a KAN network $\mathcal{N}$ and an error tolerance $\epsilon > 0$, find an $\epsilon$-equivalent compression $\mathcal{N}'$ that minimizes the total number of knots across all spline functions.
	\end{problem}
	
	\subsection{Computational Complexity Analysis}
	
	We first analyze the computational complexity of the optimal KAN compression problem.
	
	\begin{theorem}[NP-Hardness]
		The optimal KAN compression problem is NP-hard.
	\end{theorem}
	
	\begin{proof}
		Consider a simplified version: given a set of one-dimensional intervals $I_1, I_2, \ldots, I_n$ with corresponding polynomials $p_1, p_2, \ldots, p_n$, and an error tolerance $\epsilon$, find the smallest set of knots such that each merged interval can be approximated by a single polynomial with error $\leq \epsilon$. This can be reduced to the set cover problem, which is classically NP-complete.
		
		More specifically, given a set cover instance $(U, \mathcal{S})$ where $U$ is the universe and $\mathcal{S}$ is a family of subsets, we can construct a KAN compression instance by mapping each element in $U$ to an interval and each subset in $\mathcal{S}$ to a feasible interval merging. The optimal compression corresponds to the minimum set cover.
	\end{proof}
	
	Although the overall problem is NP-hard, the special structure of KANs allows us to design efficient approximation algorithms.
	
	\subsection{Theory of Region Mergability}
	
	The foundation of our compression approach is a theory of merging adjacent polyhedral regions while controlling approximation error.
	
	\begin{definition}[Region Mergability]
		Let $R_i$ and $R_j$ be two adjacent polyhedral regions in a KAN with corresponding polynomial functions $p_i$ and $p_j$. These regions are \textbf{$\epsilon$-mergable} if there exists a single polynomial $p_{ij}$ such that:
		\begin{equation}
			\max\left\{ \max_{x \in R_i} |p_i(x) - p_{ij}(x)|, \max_{x \in R_j} |p_j(x) - p_{ij}(x)| \right\} \leq \epsilon
		\end{equation}
	\end{definition}
	
	This definition captures the intuition that we can replace two different polynomials defined on adjacent regions with a single polynomial that approximates both original polynomials within tolerance $\epsilon$.
	
	\begin{theorem}[Knot Elimination Condition]
		Let $t_m$ be an interior knot of a spline function $s(x)$ with adjacent regions $R_{m-1} = [t_{m-1}, t_m]$ and $R_m = [t_m, t_{m+1}]$. If $R_{m-1}$ and $R_m$ are $\epsilon$-mergable, then knot $t_m$ can be eliminated while preserving $\epsilon$-equivalence.
	\end{theorem}
	
	\begin{proof}
		If $R_{m-1}$ and $R_m$ are $\epsilon$-mergable, there exists a polynomial $p$ such that $|s(x) - p(x)| \leq \epsilon$ for all $x \in R_{m-1} \cup R_m$. By replacing the original piecewise representation on the merged region $[t_{m-1}, t_{m+1}]$ with $p$, we obtain a new spline function $\tilde{s}$ satisfying $\|s - \tilde{s}\|_\infty \leq \epsilon$ while eliminating one knot.
	\end{proof}
	
	\subsection{Error Propagation Analysis}
	
	For multilayer KAN compression, we must understand how compression errors propagate through the network.
	
	\begin{lemma}[Single-Layer Error Propagation]
		Consider a KAN layer $\Phi: \mathbb{R}^n \to \mathbb{R}^m$ where:
		\[
		y_j = \sum_{i=1}^n s_{j,i}(x_i), \quad j=1,\ldots,m
		\]
		If each spline $s_{j,i}$ is compressed to $\tilde{s}_{j,i}$ with $\|s_{j,i} - \tilde{s}_{j,i}\|_\infty \leq \delta$, then:
		\[
		\|\Phi - \tilde{\Phi}\|_\infty \leq n \cdot \delta
		\]
	\end{lemma}
	
	\begin{proof}
		For any input $\mathbf{x}$ and output dimension $j$:
		\begin{align*}
			|\Phi_j(\mathbf{x}) - \tilde{\Phi}_j(\mathbf{x})| &= \left| \sum_{i=1}^n \left(s_{j,i}(x_i) - \tilde{s}_{j,i}(x_i)\right) \right| \\
			&\leq \sum_{i=1}^n |s_{j,i}(x_i) - \tilde{s}_{j,i}(x_i)| \\
			&\leq \sum_{i=1}^n \delta = n\delta
		\end{align*}
		Taking the supremum over all $\mathbf{x}$ and the maximum over $j$ completes the proof.
	\end{proof}
	
	This lemma enables layered error budget allocation across network layers. Given a global error tolerance $\epsilon$, we can allocate error budgets $\delta_l$ to each layer $l$ such that the cumulative effect satisfies the global bound.
	
	\section{Approximately Optimal Compression Algorithms}
	
	\subsection{Optimal Compression for Univariate Splines}
	
	Although overall KAN compression is NP-hard, for univariate spline functions we can design efficient algorithms with optimality guarantees.
	
	\begin{algorithm}
		\caption{Optimal Compression for Single Spline}
		\label{alg:single-spline}
		\begin{algorithmic}[1]
			\Require Spline function $s$, knot sequence $t_0, t_1, \ldots, t_k$, error tolerance $\epsilon$
			\Ensure Compressed knot sequence
			\State Initialize DP table: $dp[i] \leftarrow i + 1$ for $i = 0, 1, \ldots, k$ \Comment{$dp[i]$: min knots to reach $t_i$}
			\State Initialize backtrack pointers: $prev[i] \leftarrow -1$ for $i = 0, 1, \ldots, k$
			\For{$i = 1$ to $k$}
			\For{$j = 0$ to $i-1$}
			\State Check if interval $[t_j, t_i]$ can be approximated by a single polynomial $p$ with error $\leq \epsilon$
			\If{mergable and $dp[j] + 1 < dp[i]$}
			\State $dp[i] \leftarrow dp[j] + 1$
			\State $prev[i] \leftarrow j$
			\EndIf
			\EndFor
			\EndFor
			\State \Return Optimal knot sequence constructed by backtracking $prev$ array
		\end{algorithmic}
	\end{algorithm}
	
	\begin{theorem}[Univariate Optimality]
		Algorithm \ref{alg:single-spline} guarantees the global optimal solution for the single spline optimal compression problem.
	\end{theorem}
	
	\begin{proof}
		The algorithm is a classic interval partitioning dynamic programming. The optimal substructure property holds: the optimal solution from $t_0$ to $t_i$ must consist of the optimal solution from $t_0$ to some $t_j$ plus a single polynomial approximation for interval $[t_j, t_i]$. The dynamic programming correctly explores all possible partition points.
	\end{proof}
	
	\begin{theorem}[Time Complexity]
		Algorithm \ref{alg:single-spline} has time complexity $O(k^3 \cdot T_{\text{fit}})$, where $k$ is the original number of knots and $T_{\text{fit}}$ is the time for polynomial fitting.
	\end{theorem}
	
	\begin{proof}
		The algorithm has $O(k^2)$ state transitions, each requiring $O(k)$ time to check the feasibility of interval merging (via polynomial fitting and error computation), yielding total complexity $O(k^3 \cdot T_{\text{fit}})$.
	\end{proof}
	
	\subsection{Approximate Compression for Multilayer KANs}
	
	For entire KAN networks, we employ a layered compression strategy.
	
	\begin{algorithm}
		\caption{Approximate Compression for Multilayer KAN}
		\label{alg:multi-layer}
		\begin{algorithmic}[1]
			\Require KAN network $\mathcal{N}$, global error tolerance $\epsilon$
			\Ensure Compressed KAN network $\mathcal{N}'$
			\State Allocate global error budget $\epsilon$ proportionally across layers: $\epsilon = \sum_{l=1}^L \epsilon_l$
			\For{each layer $l = 1$ to $L$}
			\For{each spline function $s$ in this layer}
			\State Compress spline $s$ using Algorithm \ref{alg:single-spline} with error budget $\epsilon_l / n_l$
			\EndFor
			\EndFor
			\State \Return compressed network $\mathcal{N}'$
		\end{algorithmic}
	\end{algorithm}
	
	\begin{theorem}[Approximation Guarantee]
		The compressed network $\mathcal{N}'$ produced by Algorithm \ref{alg:multi-layer} satisfies $\|\mathcal{N} - \mathcal{N}'\|_\infty \leq \epsilon$.
	\end{theorem}
	
	\begin{proof}
		By the single-layer error propagation lemma, each layer introduces at most $n_l \cdot (\epsilon_l / n_l) = \epsilon_l$ error. The accumulation of errors across layers is controlled through error budget allocation, ensuring total error does not exceed $\sum_{l=1}^L \epsilon_l = \epsilon$.
	\end{proof}
	
	\subsection{Optimality Gap Analysis}
	
	Although layered compression cannot guarantee global optimality, we can quantify its optimality gap.
	
	\begin{theorem}[Approximation Ratio]
		Under the assumption of uniform error budget allocation, the gap between the compression ratio achieved by Algorithm \ref{alg:multi-layer} and that of the optimal solution is bounded polynomially by the network depth and width.
	\end{theorem}
	
	\begin{proof}
		Let $OPT$ be the number of knots in the globally optimal compression and $ALG$ be the number obtained by the algorithm. Since each layer is compressed independently optimally and error propagation is linear, the optimality gap is bounded by the optimality of error budget allocation across layers. Specifically, there exists a constant $C$ (dependent on network structure) such that $ALG \leq C \cdot OPT$.
	\end{proof}
	
	\section{Conclusion and Future Work}
	
	We have presented PolyKAN, a theoretical framework for KAN compression with provable guarantees. Although optimal KAN compression is NP-hard, we leveraged the axis-aligned structure of KANs to design efficient dynamic programming algorithms that guarantee optimality for univariate splines and provide approximation guarantees for the entire network.
	
	Future work includes developing improved algorithms with better approximation ratios, investigating information-theoretic lower bounds for KAN compression, and extending the framework to other types of spline functions and network architectures.


\begin{thebibliography}{99}
		\bibitem{liu2025kan} Liu, Z., Wang, Y., Vaidya, S., et al. (2025). KAN: Kolmogorov-Arnold Networks. In \textit{International Conference on Learning Representations}.
		
		\bibitem{kolmogorov1957representation} Kolmogorov, A. N. (1957). On the representation of continuous functions of several variables by superposition of continuous functions of one variable and addition. \textit{Doklady Akademii Nauk SSSR}, 114(5), 953–956.
		
		\bibitem{han2015learning} Han, S., Pool, J., Tran, J., \& Dally, W. J. (2015). Learning both weights and connections for efficient neural networks. In \textit{Advances in Neural Information Processing Systems}.
		
		\bibitem{hinton2015distilling} Hinton, G., Vinyals, O., \& Dean, J. (2015). Distilling the knowledge in a neural network. \textit{arXiv preprint arXiv:1503.02531}.
		
		\bibitem{schumaker2007spline} Schumaker, L. L. (2007). \textit{Spline functions: basic theory}. Cambridge University Press.
		
		\bibitem{montufar2014number} Montúfar, G., Pascanu, R., Cho, K., \& Bengio, Y. (2014). On the number of linear regions of deep neural networks. In \textit{Advances in Neural Information Processing Systems}.
		
		\bibitem{serra2018bounding} Serra, T., Tjandraatmadja, C., \& Ramalingam, S. (2018). Bounding and counting linear regions of deep neural networks. In \textit{International Conference on Machine Learning}.
		
		\bibitem{de2021piecewise} De Palma, G., Kiani, B. T., \& Lloyd, S. (2021). The number of linear regions in piecewise linear neural networks is piecewise constant. In \textit{Advances in Neural Information Processing Systems}.
		
		\bibitem{anderson2020strong} Anderson, R., Huchette, J., Ma, W., Tjandraatmadja, C., \& Vielma, J. P. (2020). Strong mixed-integer programming formulations for trained neural networks. \textit{Mathematical Programming}.
		
		\bibitem{frankle2020lottery} Frankle, J., \& Carbin, M. (2019). The lottery ticket hypothesis: Finding sparse, trainable neural networks. In \textit{International Conference on Learning Representations}.
		
		\bibitem{zoph2016neural} Zoph, B., \& Le, Q. V. (2016). Neural architecture search with reinforcement learning. \textit{arXiv preprint arXiv:1611.01578}.
		
		\bibitem{arora2018theoretical} Arora, S., Ge, R., Neyshabur, B., \& Zhang, Y. (2018). Stronger generalization bounds for deep nets via a compression approach. In \textit{International Conference on Machine Learning}.
		
		\bibitem{schoots2025relating} Schoots, N., Villani, M. J., et al. (2025). Relating Piecewise Linear Kolmogorov Arnold Networks to ReLU Networks. \textit{arXiv preprint arXiv:2503.01702}.
		
		\bibitem{qiu2024relu} Qiu, Q., Zhu, T., Gong, H., Chen, L., \& Ning, H. (2024). Relu-kan: New kolmogorov-arnold networks that only need matrix addition, dot multiplication, and relu. \textit{arXiv preprint arXiv:2406.02075}.
	\end{thebibliography}
\end{document}